\documentclass[sigconf,screen,authorversion,nonacm]{acmart}

\AtBeginDocument{%
  }

\usepackage{multirow}
\usepackage{comment}
\usepackage{amsmath}    
\usepackage{algorithm}  
\usepackage{algorithmic}
\usepackage{subcaption}
\usepackage{graphicx}
\usepackage{xcolor}
\usepackage[table]{xcolor}
\usepackage{caption}

\definecolor{highlightblue}{RGB}{230, 240, 255} 

\usepackage{titlesec}
\titlespacing*{\section}{0pt}{1\baselineskip}{0.1\baselineskip}
\titlespacing*{\subsection}{0pt}{0.5\baselineskip}{0.1\baselineskip}
\titlespacing*{\subsubsection}{0pt}{0.1\baselineskip}{0.1\baselineskip}

\usepackage{enumitem}

\begin{document}





\title{Retrieval-Augmented Foundation Models for Matched Molecular Pair Transformations to Recapitulate Medicinal Chemistry Intuition}

\begin{abstract}
Matched molecular pairs (MMPs) captures the local chemical edits that medicinal chemists routinely use to design analogs, but existing ML approaches either operate at the whole-molecule level with limited edit controllability or learn MMP-style edits from restricted settings and small models. 
We propose a variable-to-variable formulation of analog generation and train a foundation model on large-scale MMP transformations (MMPTs) to generate diverse variables conditioned on an input variable. 
To enable practical control, we develop prompting mechanisms that let the users specify preferred transformation patterns during generation. 
We further introduce MMPT-RAG, a retrieval-augmented framework that uses external reference analogs as contextual guidance to steer generation and generalize from project-specific series. 
Experiments on general chemical corpora and patent-specific datasets demonstrate improved diversity, novelty, and controllability, and show that our method recovers realistic analog structures in practical discovery scenarios.
\end{abstract}


\author{Bo Pan}
\affiliation{%
 \institution{Department of Computer Science, Emory University}
 \city{Atlanta}
 \state{GA}
 \country{USA}}
 \email{bo.pan@emory.edu}

\author{Peter Zhiping Zhang}
\affiliation{%
 \institution{Merck \& Co., Inc.}
 \city{Rahway}
 \state{NJ}
 \country{USA}}
 \email{zhiping.peter.zhang@merck.com}

\author{Hao-Wei Pang}
\affiliation{%
 \institution{Merck \& Co., Inc.}
 \city{Rahway}
 \state{NJ}
 \country{USA}}
 \email{hao-wei.pang@merck.com}

\author{Alex Zhu}
\affiliation{%
 \institution{Department of Computer Science, Emory University}
 \city{Atlanta}
 \state{GA}
 \country{USA}}
 \email{alex.zhu@emory.edu}

\author{Xiang Yu}
\affiliation{%
 \institution{Merck \& Co., Inc.}
 \city{Rahway}
 \state{NJ}
 \country{USA}}
 \email{xiang.yu2@merck.com}
 
\author{Liying Zhang}
\affiliation{%
 \institution{Merck \& Co., Inc.}
 \city{Rahway}
 \state{NJ}
 \country{USA}}
 \email{liying.zhang@merck.com}

\author{Liang Zhao}
\affiliation{%
 \institution{Department of Computer Science, Emory University}
 \city{Atlanta}
 \state{GA}
 \country{USA}}
 \email{liang.zhao@emory.edu}

\renewcommand{\shortauthors}{Pan, Zhang, Zhao et al.}







\maketitle
\begin{figure}[t]
    \centering
    \vspace{3mm}\hspace{-2mm}\includegraphics[width=1.0\linewidth]{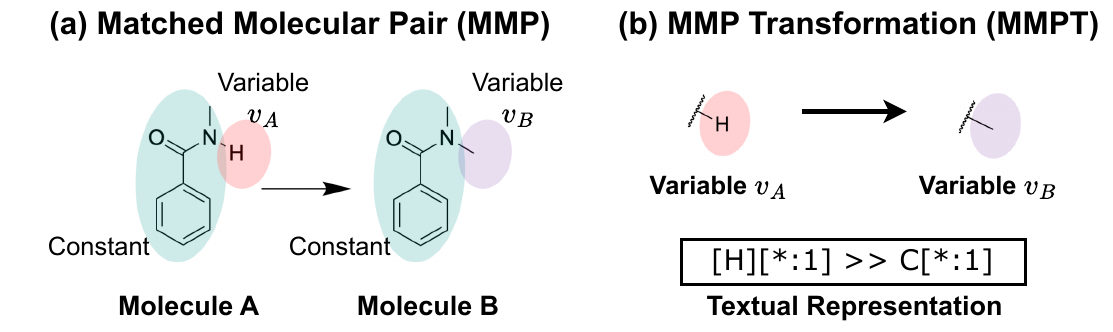}
    \caption{An example of (a) Matched Molecular Pairs (MMP); (b) Matched Molecular Pair Transformation (MMPT) and its textual representation.}
    \vspace{-6mm}
    \label{fig:mmp}
\end{figure}

\section{Introduction}

{In drug discovery, a fundamental strategy to optimize lead molecules is \textit{analog design}, which involves medicinal chemists leveraging their intuition to make localized, knowledge-driven edits to existing molecules, instead of designing entirely novel molecules \citep{fischer2010analogue, choung2023extracting, shurtleff2024invention}. Machine learning models for molecular optimization have increasingly adopted transformation-based formulations, learning to transform one molecule into another through graph edits or sequence-to-sequence generation \citep{tysinger2023can, tibo2024exhaustive, he2022transformer, he2021molecular}. However, most such approaches treat transformations as implicit, molecule-level operations, without distinguishing which edits correspond to chemically meaningful local modifications versus arbitrary global rewrites \citep{ozcelik2025generative}. In contrast, medicinal chemists typically reason in terms of \textbf{matched molecular pairs (MMPs)} \cite{dossetter2013matched, yang2023matched, fischer2010analogue, choung2023extracting, andreev2020discovery, meanwell2017synopsis, papadatos2010lead, griffen2011matched}, which are pairs of molecules that differ by discrete and minimal modifications, such as R-group substitutions or core replacements, that preserve synthetic feasibility and support interpretable structure–activity comparisons, as shown in Fig.~\ref{fig:mmp}~(a).  
Abstracting away the constant chemical context and focusing solely on the localized edit yields an \textbf{MMP transformation (MMPT)}, which captures a context-independent medicinal chemistry modification corresponding to a single, well-defined \textit{variable-to-variable} change, as an example shown in Fig.~\ref{fig:mmp}~(b). MMPTs directly reflect how chemists use their medicinal chemistry intuition to optimize lead molecules, and provide a principled way to represent and learn transferable medicinal chemistry modifications across different molecular contexts.

While MMPTs have appeared occasionally in machine learning-based molecular optimization, they remain underexploited as a first-class generative representation: prior models are typically limited in scale or trained with specific data \citep{chen2021molecule, jin2020multi}, or embedded within molecule-level generators where MMP relationship is only weakly enforced \citep{tysinger2023can, tibo2024exhaustive, he2022transformer, he2021molecular}. At the same time, recent advances in large-scale chemical corpora~\cite{gaulton2012chembl, sun2017excape, yu2024llasmol} and foundation models~\cite{achiam2023gpt, bai2023qwen} make it newly feasible to learn generalizable, MMPT-level priors directly from millions of real data points. This convergence creates a timely opportunity to revisit MMPTs: not as an auxiliary constraint, but as the central abstraction for controllable and scalable analog generation.}

Despite their conceptual appeal, realizing MMPT-centric learning within modern ML systems poses several technical challenges. First, many existing analog generation models are optimized for molecule-level similarity rather than explicit localized modifications~\citep{tibo2024exhaustive, he2022transformer}, making it difficult to guarantee that generated candidates differ from the input by a single, well-defined transformation. As a result, users cannot reliably specify where a modification should occur (e.g., which R-group or core), nor ensure that unintended global changes are avoided. Second, prior MMPT-based learning approaches are often constrained to small models, limited transformation vocabularies, or narrowly curated datasets, which restrict their ability to absorb the long-tailed, heterogeneous transformation patterns observed across large chemical corpora. Third, existing controllable methods, whether graph-editing models with fixed operators \citep{chen2021molecule, jin2020multi} or prompt-based use of large language models \citep{wu2024leveraging, wang2024efficient, brown2020language}, lack mechanisms to learn transferable, transformation-level priors that generalize across scaffolds while remaining interpretable and synthetically plausible. Finally, practical deployment in medicinal chemistry requires models to adapt to user- or project-specific preferences, such as emphasizing rare but relevant modifications or following established series patterns, without costly retraining. Addressing these challenges requires moving beyond molecule-level generators toward scalable, controllable models that operate directly in the MMPT space.

Motivated by the above complementarity and limitations, we aim to build a practical MMPT-centric generation framework that synergizes chemists' intuition with the ability of modern ML to learn from massive data. First, we train an MMPT foundation model on large-scale transformation data to generate variables conditioned on input variables, enabling high-throughput analog design in a user-controlled edit region.
Second, we develop prompting mechanisms that expose substructure-level control to users, allowing them to customize the structural patterns of generated variables.
Third, we introduce an MMPT-RAG framework that incorporates external reference datasets: retrieved analogs are organized into structural clusters and used to reweight the generation distribution, improving coverage of infrequent yet chemically meaningful transformations while preserving plausibility.

Our main contributions are summarized as follows:
\begin{itemize}
    \item We formalize analog generation in the matched molecular pair transformation (MMPT) space, treating analog design as context-independent local edits that can be composed across diverse molecular scaffolds.
    \item We train a foundation model on large-scale MMPTs extracted from a broad corpus of drug-like molecules, and enable controllable generation via prompting mechanisms that specify desired transformation templates or structural patterns.
    \item We propose an MMPT-RAG framework that leverages external reference datasets by retrieving structurally related examples and using them as contextual guidance to steer generation toward user-preferred patterns.
    \item We validate our approach on three complementary MMPT benchmarks: an in-distribution setting, a within-patent analog expansion setting, and a cross-patent generalization setting. Across tasks, our method consistently improves recall of ground-truth transformations while maintaining high validity and producing non-trivial novel edits.
\end{itemize}

\section{Related Work}

\subsection{MMP-Based Analog Generation}

Most of the existing MMP-based analog design methods explore the setting of molecule-level generation, i.e., generating entire molecules that are similar to the given molecule
\citep{jin2018learning, tysinger2023can, tibo2024exhaustive, he2022transformer, he2021molecular}, and thus this stream of methods does not support users specifying a substructure (variable) to edit. Some existing methods also allow users to specify a structure to change, with one stream of work formulating it as a partial molecule generation problem, where the user-provided substructure is fixed, and the model is tasked with completing the remaining molecule, including those VAE-based graph generative models which operate via node generation \cite{jin2018learning, jin2020multi, lim2020scaffold, you2018graph} and auto-regressive generative models that implements it as a token generation task in the SMILES space \cite{olivecrona2017molecular, loeffler2024reinvent}. With the advancements of large language models (LLMs), some work also explored leverate LLMs' zero-shot ability to suggest variable replacements \citep{wu2024leveraging, wang2024efficient, brown2020language}. Among these methods, the LibINVENT module \citep{loeffler2024reinvent} in REINVENT 4 \cite{loeffler2024reinvent} stands out as a strong baseline with a wide industrial applicability, benefiting from its large training corpus. However, all the above methods are trying to learn the conditioned molecular completion task (constant to variable). To our best knowledge, no existing method tries to directly learn a context-independent variable replacement objective (variable to variable).

\subsection{Retrieval-Augmented Generation for Molecule Generation}

Recent advances in retrieval-augmented generation (RAG) have been applied with notable success to molecular and materials design. For example, RetMol \citep{wang2022retrieval} uses an exemplar retrieval module to guide a pretrained molecular generator by fusing input compounds with retrieved analogues, enabling efficient design of molecules satisfying complex properties without task-specific fine-tuning. \cite{lee2024molecule} propose f-RAG, which retrieves both “hard” and “soft” fragment contexts to steer a fragment-based generative model, thereby improving diversity and design novelty. Structure-based methods such as Rag2Mol \citep{zhang2025rag2mol} and IRDiff \citep{huang2024interaction} integrate retrieval of known ligands or fragments and inject them into 3D generation via autoregressive graph or diffusion models, aligning generation to target binding pockets. \cite{xu2025reimagining} further extend this paradigm by introducing READ, an SE(3)-equivariant diffusion model aligned with retrieval of scaffold embeddings to enhance geometric and chemical realism. Finally, \cite{krotkov2025nanostructured} demonstrate the flexibility of RAG techniques in materials science, coupling literature retrieval with LLM-based generative reasoning for nanostructured material design.

\section{Problem Formulation}
\subsection{MMPTs as the Generative Unit}
In medicinal chemistry, analog design proceeds by modifying an existing compound through a single, localized structural change, such as replacing a substituent, linker, or core, while keeping the remainder of the molecule fixed. 
This concept is formalized through matched molecular pair transformations (MMPTs). An MMPT is defined as the transformation linking a pair of molecules that share an invariant chemical context, while differing by a single, well-defined variable fragment, as illustrated in Fig.~\ref{fig:mmp}.
Formally, an MMPT can be represented as
$(v_A\to v_B)$,
where $v_A$, $v_B$ are referred to as the \textbf{\textit{variables}} before and after the transformation. By construction, MMPTs isolate minimal chemical edits that are synthetically feasible and empirically validated through historical discovery efforts. As such, MMPTs constitute the primary unit through which medicinal chemists reason about structure–activity relationships (SAR) and explore local chemical space.

\subsection{Problem Definition: MMPT-Centric Analog Generation}
Given an input variable $v_A$, our goal is to generate chemically plausible alternative variables that correspond to valid MMPTs.
Concretely, given an input variable $v_A$, we aim to generate a set of candidate substitutions
$\{v_B^{(1)}, v_B^{(2)}, \dots\}$ such that each pair $(v_A \rightarrow v_B^{(i)})$ constitutes a valid MMPT. 

This formulation differs fundamentally from molecule-level analog generation, where edits are implicit and entangled across the structure. By operating directly in the MMPT space, the task becomes learning conditional distributions over chemical transformations, rather than over entire molecules.

Although MMPTs provide a natural abstraction for localized analog design, learning to generate MMPTs at scale poses several challenges. First, the space of MMPT is large, sparse, and highly imbalanced: a small number of common substitutions dominate, while many chemically meaningful transformations occur rarely, making it difficult for task-specific or small models to learn transferable priors. 
Second, models grounded in formal chemical languages preserve precise structural semantics but are inherently rigid and difficult to instruct through user prompts; on the other hand, natural language models support flexible prompting, but often lack explicit structural constraint enforcement and are not specialized for chemical data. It is fundamentally challenging to retain chemically grounded priors while achieving user controllability.
Third, practical deployment in medicinal chemistry requires explicit distributional steering, enabling users to bias generation toward preferred or project-specific MMPT patterns without retraining.

\begin{figure*}[t]
    \centering
    \includegraphics[width=.8\linewidth]{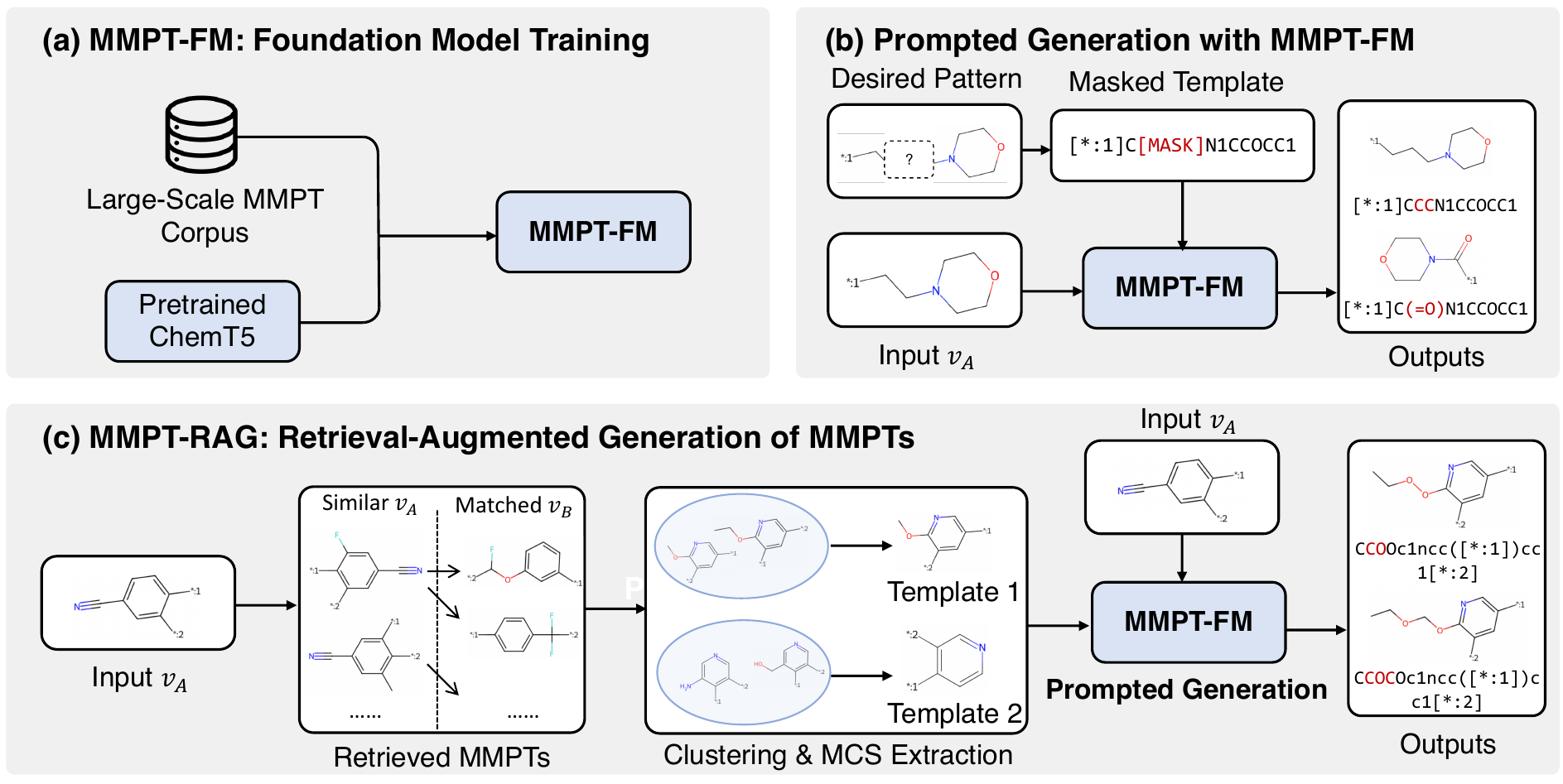}
    \vspace{-3mm}
    \caption{Overview of the proposed MMPT framework. (a) The foundation model (MMPT-FM) is trained on large-scale MMPT data.
(b) MMPT-FM supports controllable generation via masked template prompting.
(c) MMPT-RAG augments generation with retrieval, clustering, and MCS-based template extraction to guide context-aware transformation generation.}
\vspace{-3mm}
    \label{fig:main}
\end{figure*}

\section{Methodology}

To address the challenges outlined above, we propose a two-component framework centered on MMPTs. First, to address the challenges of the large, long-tailed MMPT search space and the challenge in enabling flexible user control, we introduce MMPT-FM, a foundation model trained directly on large-scale MMPT data. By modeling variable-to-variable transformations in a chemically grounded token space and supporting structured prompting, MMPT-FM learns generalizable transformation-level priors while enabling user-guided generation. To address the challenge of explicit, project-level controllability without retraining, we introduce MMPT-RAG, a retrieval-augmented generation framework that incorporates external reference analogs as guidance, allowing users to bias generation toward preferred transformation patterns and project-specific needs.

\subsection{MMPT-FM: A Promptable MMPT Foundation Model}\label{sec:fm}



Here, we construct our foundation model, MMPT-FM, that operates directly in the MMPT space, learning variable-to-variable transformations as the primary generative unit. We further equip it with a prompting mechanism that conditions generation on optional user-specified structural templates, thereby enabling more flexible user control.

\subsubsection{Training of MMPT-FM.}

As described earlier, we train our foundation model on the variables rather than on entire molecules, modeling each MMPT as a localized transformation $v_A\to v_B$, where both $v_A$ and $v_B$ are represented as SMARTS~\cite{daylight2019smarts}, a textual chemical representation method for substructures. 
Formally, let $\Sigma$ denote a finite vocabulary of SMARTS tokens, including atomic symbols (e.g., C, N, O), bond descriptors (e.g., =, \#), and other syntactic elements for SMARTS expressions. 
Each variable is represented as a sequence over $\Sigma$ as $v_A = (\tau_1^{A}, \dots, \tau_{n}^{A}),
v_B = (\tau_1^{B}, \dots, \tau_{m}^{B})$, 
where $\tau_i^{A}, \tau_j^{B} \in \Sigma$.
An MMPT is therefore formulated as a conditional sequence-to-sequence mapping from $v_A$ to $v_B$, and the model learns the conditional distribution
$p(v_B \mid v_A)$
defined over syntactically valid and chemically plausible SMARTS variables.

To construct the training data, we use MMPDB \cite{dalke2018mmpdb} to extract MMPs from ChEMBL, a large database of drug-like compounds \cite{gaulton2012chembl}. To enrich for drug-like chemistry, we first filter molecules using the medchem package’s \texttt{rule\_of\_druglike\_soft} criterion \cite{Noutahi_medchem_2025}, further retain only compounds with molecular weight $\geq 200$ Da, and remove structural alerts using the curated list compiled by Walters \cite{PatWalters_alert_list_2018}. After filtering, 800,714 compounds remain. We then generate MMPs using MMPDB with the maximum variable ratio constrained to $33\%$ (\texttt{max-variable-ratio}=0.33), yielding 2.63 million MMPs.
For each MMP, MMPDB identifies the shared substructure (the \emph{constant} part) and the differing variable pair (the \emph{variable} parts). 
Totally, there are approximately 800K distinct MMPTs. We use $90\%$ of these transformations for training and reserve the remaining $10\%$ as a held-out evaluation set.

To model this variable-to-variable translation task, we adopt an encoder–decoder Transformer architecture. The encoder takes as input the variable that the user wishes to replace, and the decoder generates suggested replacement variables. To provide the model with a chemically aware initialization, we initialize from T5Chem~\citep{christofidellis2023unifying, lu2022unified}, a T5-style model pretrained on large-scale chemical tasks, which offers a moderate model size and strong coverage of chemical syntax and semantics.

\subsubsection{Prompted Generation with MMPT-FM.}\label{sec:prompt}

In practical medicinal chemistry workflows, analog design is rarely unconstrained; chemists often seek to impose explicit structural conditions, such as preserving a core motif or exploring a specific transformation family, making it essential for an MMPT generation framework to support user-specified structural patterns.
Unlike natural language models, where intent can be expressed through appended text instructions, MMPT-FM operates purely in chemical token space, so desired structural patterns cannot be specified as free-form text. To preserve the variable-to-variable formulation while enabling control, we encode user intent as partial structural constraints on the output.

The prompted generation process is illustrated in Fig.~\ref{fig:main}~(b). Specifically, we formulate prompted generation of MMPT-FM as a constrained sequence completion task, where the user can impose a structural constraint $S$ that defines the preferred chemical substructure to be preserved in the generated $v_B$, with some undefined positions that the model is asked to complete.
Formally, we define a masked template $T$ as
$T = (\tau_1, \tau_2, \dots, \tau_L), \text{where each } \tau_i \in \Sigma \cup \{\texttt{[MASK]}\}$. The mapping from a user-specified structural constraint $S$ to a textual template $T$ is operated by preserving the chemical tokens corresponding to the fixed substructure $S$ while replacing the undefined positions with $\texttt{[MASK]}$ tokens.

The generation objective is to find complete variable sequences $v_B$ that are {compatible} with $T$. 
At inference time, the generator performs approximate inference to produce a set of candidate variables
$\mathcal{V}_B = \{v_b^{(1)}, \dots, v_b^{(K)}\}$
that complete the masked positions conditioned on ${T}$.
These candidates are selected to have high likelihood under the model as
\begin{equation}
\text{PromptGen}(v_A, \tilde{T}, K) = \underset{\mathcal{S}, |\mathcal{S}|=K}{\text{argmax}}
\sum_{v_B^{(i)} \in \mathcal{S}}
\log
p_{\theta}
\!\left(
v_B^{(i)} \mid v_A, {T}
\right),
\label{eq:infilling-template}
\end{equation}
where $\theta$ denotes the parameters of the trained MMPT-FM, and $K$ is the desired number of output variables.

We implement masked infilling via an explicit search over the space of possible span completions, using the model likelihood to score each candidate. To make this search tractable, in practice, we limit the branching factor at each masked position by selecting only the effective number $N_{\text{eff}}$~\cite{jost2006entropy} of high-probability tokens, calculated by $N_{\text{eff}} = 2^{H_2(p)}$, where $p$ is the token probability distribution of generating each token, and $H_2(p) = -\log_2 \sum_i p_i^2$ is the Rényi entropy~\cite{renyi1961measures} of order 2 of the token distribution $p$. Starting from the initial masked template, we conduct a tree search by iteratively filling masked positions with these candidate tokens, and accumulating sequence-level log-likelihoods from the model. Finally, we rank completed candidates by their likelihood and return the top-$k$ infilled fragments as the model’s suggested variable substitutions.

\subsection{MMPT-RAG: A RAG Framework for MMPT Generation}\label{sec:rag}

MMPT-FM learns a global prior over chemically meaningful MMPTs from large-scale data, whereas practical design often relies on project-specific reference analogs. A principled generator should therefore interpolate between general chemical knowledge and project-specific patterns, rather than overriding one with the other. We realize this idea through MMPT-RAG, which integrates retrieval as guidance (Section~\ref{sec:rag-workflow}) and admits a theoretical interpretation as an explicit distribution shift from model prior to reference dataset distribution (Section~\ref{sec:rag-theory}). 

\subsubsection{Workflow. }\label{sec:rag-workflow}
MMPT-RAG guides the generation towards a reference database $\mathcal{D}$. Let $\mathcal{D} = \{(v_A^{(i)}\to v_B^{(i)})\}_{i=1}^{N}$ denote a reference database of MMPTs. Given an input variable $v_A$, the framework retrieves relevant MMPTs, clusters them to identify representative patterns, and converts each cluster into a structural template. These templates are then used to prompt MMPT-FM via masked infilling, as illustrated in Fig.~\ref{fig:main}~(c). The workflow proceeds in three steps detailed below.

\textit{Step 1: Retrieval with input similarity. }
To leverage the useful examples from the retrieval dataset, we retrieve variables structurally similar to the query. 
Let $\psi(\cdot)$ denote an embedding function, which is usually implemented with the Morgan fingerprint \cite{morgan1965generation}. 
For query $v_A$, we retrieve top-$K$ nearest neighbors $\mathcal{V}_A$ from the database $\mathcal{D}$ as
\begin{equation}\begin{aligned}
    \mathcal{V}_A
=&\operatorname{Retrieve}\!\left(v_A; \mathcal{D}\right)\\=&
\left\{
v_A^{(i)}
\;\middle|\;
(v_A^{(i)}\to v_B^{(i)})\in\mathcal{D}, i \in \operatorname{TopK}
\left(
\text{sim}(\psi(v_A), \psi(v_A^{(i)}))
\right)
\right\},
\end{aligned}
\end{equation}
where sim($\cdot$) denotes a function to calculate similarity, usually implemented using cosine similarity,
yielding candidate contexts $v_A'\in \mathcal{V}_A$. Given a retrieved set of input variables $\mathcal{V}_A=\{v_A^{(i)}\}$, we denote by $\mathcal{V}_B=\{v_B^{(j)}\}$ the set of all variables such that there exists some $i, j$ that $(v_A^{(i)} \rightarrow v_B^{(j)})$ forms an MMPT in $\mathcal{D}$ as
\begin{equation}
  \mathcal{V}_B = \left\{ v_B \mid \exists v_A \in \mathcal{V}_A, (v_A, v_B) \in \mathcal{D} \right\}  
\end{equation}

\textit{Step 2: Clustering of Retrieved Examples.}
To extract representative structural patterns in retrieved $\mathcal{V_B}$, we perform clustering over the retrieved outputs.
We compute embeddings $\phi(v_B)$ for all $v_B \in \mathcal{V_B}$ and partition them into $K$ clusters
\begin{equation}
\mathcal{C}_1, \ldots, \mathcal{C}_K
\;=\;
\operatorname{Cluster}\!\left(\{\,\phi(v_B): v_B\in\mathcal{V}_B\,\}\right),
\end{equation}
using a clustering algorithm, which is implemented as HDBSCAN \cite{mcinnes2017hdbscan} in this work.

\textit{Step 3: Cluster pattern-prompted generation.}
For each cluster $\mathcal{C}_k$, we first identify its invariant chemical scaffold.  
The invariant substructure is defined as the Maximum Common Substructure (MCS) \cite{raymond2002rascal} among these retrieved molecules as
$S_k = \operatorname{MCS}(\mathcal{C}_k)$. 
In practice, this can be automatically computed using the \texttt{rdkit.Chem.MCS.FindMCS} function \cite{landrum2013rdkit}, which finds the largest subgraph common to all variables in $\mathcal{C}_k$. Then we construct the masked template $T_k$ for each MCS $S_k$ via the approach introduced in Section~\ref{sec:prompt}.
The resulting $T_k$ serves as the cluster-invariant template, which is further used as the prompt to guide the generation.

Given cluster-specific templates $\{{T}_k\}_{k=1}^K$ and an (optional) user-specified preference distribution
$\tilde{\boldsymbol{\pi}}(x)$, MMPT-RAG generates outputs from each template proportionally to its assigned weight.
Formally, the RAG output is defined as
\begin{equation}
\text{RAG}(v_A)
=
\bigcup_{k=1}^K
\operatorname{PromptGen}\!\left(v_A, {T}_k,\, N_k\right),
\label{eq:rag-gen}
\end{equation}
where $N_k$ denotes the generation budget allocated to cluster $k$ and $\qquad
N_k \propto \tilde{\pi}_k(x),$ $\tilde{\boldsymbol{\pi}}(x)$ is the user-specified preferred cluster distribution for each cluster $C_k$:
\begin{equation}
\tilde{\boldsymbol{\pi}}(x)
=
\big(\tilde{\pi}_1(x), \ldots, \tilde{\pi}_K(x)\big),
\qquad
\tilde{\pi}_k(x) \ge 0,\quad
\sum_{k=1}^K \tilde{\pi}_k(x) = 1,    
\end{equation}
where $\tilde{\pi}_k(x)$ can encode arbitrary user preferences.

\subsubsection{Theoretical Analysis. }\label{sec:rag-theory}

To analyze the RAG mechanism, we formalize the interaction between the foundation model and the reference set as a Bayesian integration. We show that MMPT-RAG performs a global distribution shifts toward the reference dataset while maintaining the knowledge of the foundation model

\begin{theorem}[Global Steering]\label{thm:steering}
Let $p_{\theta}(y \mid x)$ be the conditional distribution over variables $y \in \mathcal{V}$ defined by the unconstrained foundation model. 
Assume that for each cluster $k$, prompting the model with template $T_k$ (via masked infilling) results in a local distribution $p(y \mid x, T_k)$ that is an adaptive interpolation between the model's prior and the cluster-specific reference $p(y \mid T_k)$:
\begin{equation}
p(y \mid x, T_k) = (1 - \alpha_k) p_{\theta}(y \mid x) + \alpha_k p(y \mid T_k),
\label{eq:local_shift}
\end{equation}
where $\alpha_k \in (0,1]$ is an \emph{adaptive gating factor} reflecting the model's adherence to template $T_k$ under context $x$. 

Then, the global RAG distribution defined in \eqref{eq:rag-gen}, $p_{\text{RAG}}(y \mid x)$, satisfies:
\begin{equation}
p_{\text{RAG}}(y \mid x) = (1 - \bar{\alpha}) p_{\theta}(y \mid x) + \bar{\alpha} p_{\text{ref}}^*(y \mid x),
\end{equation}
where $\bar{\alpha} = \sum_{k} \tilde{\pi}_k \alpha_k$ and $p_{\text{ref}}^*(y \mid x) = \sum_{k} \frac{\tilde{\pi}_k \alpha_k}{\bar{\alpha}} p(y \mid T_k)$. 
\end{theorem}
\textbf{{Proof of Theorem~\ref{thm:steering}.}} See Appendix~\ref{append:proof}.

Theorem~\ref{thm:steering} shows that the RAG distribution is a convex interpolation between the original foundation model distribution and a reference set distribution. $\bar{\alpha}$ serves as a weight that quantifies the distribution shift.

\section{Experiments}
In this section, we evaluate our framework systematically through three progressively difficult tasks, which range from generic in-distribution MMPT generation to the prediction of future analogs in subsequent patents (Section~\ref{sec:main_result}). Following these main results, we provide a decoupled analysis that examines how the model covers chemical space (Section~\ref{sec:space_coverage}) and follows user prompts (Section~\ref{sec:eval_prompt}) while also demonstrating the ability of retrieval to align generations with specific project domains (Section~\ref{sec:eval_rag}). The section concludes with a study of hyperparameter sensitivity (Section~\ref{sec:hyperparam}) and a qualitative review of specific chemical transformations to illustrate the practical utility of the framework (Section~\ref{sec:qualitative}).

\subsection{Experimental Setup}\label{sec:setup}
\subsubsection{Main Experiment Tasks. }We evaluate our framework from the perspective of three progressively more realistic analog-generation tasks, each instantiated with a corresponding dataset.

\textbf{Task 1: In-distribution MMPT Generation.}
The first task evaluates whether the model can recover and generate valid and novel local transformations under an in-distribution setting.
We instantiate this task using the 10\% held-out test split from the ChEMBL MMPT dataset, constructed with the same MMPT extraction pipeline as training but with disjoint MMPTs.

\textbf{Task 2: Within-Patent Analog Expansion.}
The second task evaluates MMPT generation within a real-world medicinal chemistry project. We construct this setting using the PMV Pharmaceutical patent dataset (PMV17)~\cite{vu2017methods} with MMPDB~\cite{dalke2018mmpdb} to extract MMPTs.
This task evaluates whether the model can discover promising variables in a realistic setting.

\textbf{Task 3: Cross-Patent Follow-up Generation.}
The third task evaluates whether a model can propose forward-looking MMPTs that may appear in later patents, a realistic and challenging evaluation of temporal medicinal chemistry progression.
We construct a patent-to-patent setting by extracting MMPTs linking compounds from PMV17 (from 2017) to those appearing in subsequent patents (PMV21)~\cite{vu2021methods} (from 2021), both derived from patent filings by PMV Pharmaceuticals, Inc.

\subsubsection{Compared Methods. }
To the best of our knowledge, no existing method is explicitly designed to operate in the variable-to-variable MMPT formulation. The only directly comparable baseline in the MMPT space is similarity-based \textbf{database retrieval}, which is a non-learning method that returns nearest-neighbor variables from the reference dataset. To further situate our results within established industrial practice, we additionally include \textbf{REINVENT4} (LibINVENT module) \cite{loeffler2024reinvent}, a state-of-the-art molecule-level analog generation framework. Although we acknowledge that LibINVENT operates on a different objective by conditioning on a fixed constant scaffold rather than the variable part, we still compare with it to demonstrate that our method can perform better even without the auxiliary information. We report both 
\textbf{MMPT-FM}, which only generates with our foundation model, and
\textbf{MMPT-RAG}, which denotes the full proposed RAG framework. 

\subsubsection{Evaluation Metrics.}
We report a consistent set of metrics across tasks to assess validity, novelty, and recoverability of medicinal-chemistry transformations.
\textbf{Valid} measures the percentage of generated strings that form chemically valid variables and have the same number of attachment points as the input.
\textbf{Novel} reports the percentage of generated variables not seen during training. Specifically, \textbf{Novel/valid} is calculated by the number of novel and valid variables divided by the number of valid variables; \textbf{Novel/all} is calculated by the number of novel and valid variables divided by the number of all outputs.
\textbf{Recall} measures the percentage of ground-truth variables recovered by the model.
For patent-based tasks, we further report \textbf{Recall-i} and \textbf{Recall-o}, which measure recovery of in-training-set and out-of-training-set transformations, respectively. Among them, Novel and Recall-o are the two most important metrics to evaluate the models' performance since both novelty and ability to learn from prior knowledge are critical in generating analogs by mimicking medicinal chemists' intuition.

\subsubsection{Implementation Details. } 
All implementation details are given in Appendix~\ref{append:implementation}.

\subsection{Main Results}\label{sec:main_result}
\subsubsection{Task 1: In-distribution Evaluation on ChEMBL}\label{sec:task1}

Table~\ref{table:chembl} reports the results on Task 1. As expected, database retrieval achieves moderate recall (43.5\%) but yields no novel outputs, reflecting the inherent limitation of exact analog lookup. REINVENT4 attains higher novelty (23.0\%) but suffers from very low recall (12.7\%), indicating that unconstrained molecule-level generation struggles to reproduce specific, localized MMP edits even in an in-distribution setting.
In contrast, MMPT-FM substantially improves recall to 67.6\% while maintaining high validity. Building on this foundation, MMPT-RAG further boosts recall to 82.1\% and achieves the highest novelty (30.1\%) among all methods. This improvement confirms that MMPT-RAG is a strong complement for the foundation model itself by leveraging less-represented but still related MMPT patterns.
Overall, these results show that MMPT-centric modeling is effective in in-distribution MMPT generation, and that retrieval augmentation further strengthens coverage of valid transformation patterns beyond what can be achieved by learning alone.

\begin{table}[th]
\centering
\caption{MMPT generation performance on three tasks.}
\vspace{-2mm}
\begin{subtable}{\linewidth}
\centering
\small
\caption{Task 1: ChEMBL MMPT Dataset}
\vspace{-2mm}
\label{table:chembl}
\begin{tabular}{lcccc}
\toprule
\textbf{Method} & \textbf{Recall} & \textbf{Novel/valid} & \textbf{Novel/all} & \textbf{Valid} \\
\midrule
Database Retrieval & 43.5\% & 0.0\% & 0.0\% & 100\% \\
REINVENT 4         & 12.7\% & 23.0\% & 5.6\% & 24.4\% \\
\rowcolor{highlightblue} MMPT-FM (Ours)  & 67.6\% & 26.0\% & 25.8\% & 99.3\% \\
\rowcolor{highlightblue} MMPT-RAG (Ours) & 82.1\% & 30.1\% & 29.8\% & 99.1\% \\
\bottomrule
\end{tabular}
\end{subtable}
\vspace{2mm}
\begin{subtable}{\linewidth}
\centering
\small
\caption{Task 2: PMV17 MMPT Dataset}
\vspace{-2mm}
\hspace{-5mm}
\label{table:pmv17}
\resizebox{1.07\linewidth}{!}{%
\begin{tabular}{lcccccc}
\toprule
\multirow{2}{*}{\textbf{Method}} & \multicolumn{3}{c}{\textbf{Recall}} & \multicolumn{2}{c}{\textbf{Novel}} & \multirow{2}{*}{\textbf{Valid}} \\
\cmidrule(lr){2-4} \cmidrule(lr){5-6}
 & \textbf{Total} & \textbf{Recall/i} & \textbf{Recall/o} & \textbf{/valid} & \textbf{/all} & \\
\midrule
Database Retrieval & 22.7\% & 29.4\% & 0.0\% & 0.0\% & 0.0\% & 100\% \\
REINVENT 4         & 5.1\%  & 7.1\%  & 0.0\% & 48.2\% & 15.7\% & 32.5\% \\
\rowcolor{highlightblue} MMPT-FM (Ours)  & 41.4\% & 52.1\% & 13.2\% & 23.0\% & 22.7\% & 98.9\% \\
\rowcolor{highlightblue} MMPT-RAG (Ours) & 49.2\% & 62.1\% & 15.2\% & 23.7\% & 23.4\% & 98.6\% \\
\bottomrule
\end{tabular}}
\end{subtable}

\begin{subtable}{\linewidth}
\centering
\small
\caption{Task 3: PMV17-PMV21 Cross-Patent Generation}
\vspace{-2mm}
\label{tab:pmv21_only}
\begin{tabular}{lccc}
\toprule
\textbf{Method} & \textbf{Recall} & \textbf{Recall/i} & \textbf{Recall/o} \\
\midrule
Database Retrieval  & 28.57\% & 57.49\% & 0.00\% \\
REINVENT 4          & 7.36\%  & 12.21\% & 1.87\% \\
\rowcolor{highlightblue} MMPT-FM (Ours)   & 43.77\% & 76.45\% & 11.48\% \\
\rowcolor{highlightblue} MMPT-RAG (Ours)  & 46.81\% & 81.35\% & 12.99\% \\
\bottomrule
\end{tabular}
\end{subtable}
\end{table}

\subsubsection{Task 2: Within-Patent Analog Expansion on PMV17.}\label{sec:task2}

Table~\ref{table:pmv17} reports results on Task 2. As in Task~1, database retrieval achieves limited recall (22.7\%) and completely fails to recover structurally novel transformations, highlighting that exact lookup is insufficient for realistic series expansion. REINVENT4 exhibits very low recall across all metrics.
In contrast, MMPT-FM substantially improves overall recall to 41.36\% while achieving strong in-training-set recovery (Recall-i = 52.06\%) and non-trivial out-of-training-set recall (Recall-o = 13.15\%). 
MMPT-RAG further improves performance across all metrics, achieving the highest overall recall (49.21\%), the strongest in-training-set recovery (62.08\%), and the best out-of-training-set recall (15.24\%). 
The gains in Recall indicate that MMPT-RAG effectively helps guide the generator toward a region which is closer to the PMV17 dataset.

\subsubsection{Task 3: Cross-Patent Generation (PMV17 $\rightarrow$ PMV21)}\label{sec:task3}


Table~\ref{tab:pmv21_only} reports the results for Task 3.
Following the same pattern, database retrieval fails entirely on out-of-training transformations, and
REINVENT4 exhibits extremely low recall across all metrics. MMPT-FM substantially improves recall by modeling transformation-level priors, achieving a recall of 43.77\% and recovering a large fraction of in-training transformations.
By incorporating retrieval-augmented prompting, MMPT-RAG further improves performance across all metrics, achieving the highest overall recall (46.81\%), in-training recall (81.35\%), and out-of-training recall (12.99\%).

\subsection{Decoupled Analysis}
\subsubsection{Evaluation of MMPT-FM's Chemical Space Coverage}\label{sec:space_coverage}
\begin{figure}[t]
    \centering
    \begin{subfigure}[b]{0.23\textwidth}
        \centering
        \includegraphics[width=\textwidth, trim={20mm 8mm 19mm 8mm}, clip]{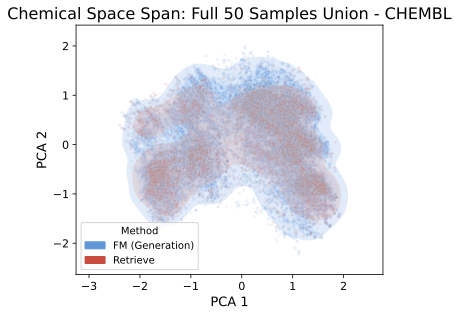} 
        \caption{ChEMBL dataset.}
    \end{subfigure}
    \begin{subfigure}[b]{0.23\textwidth}
        \centering
        \includegraphics[width=\textwidth, trim={20mm 8mm 17mm 8mm}, clip]{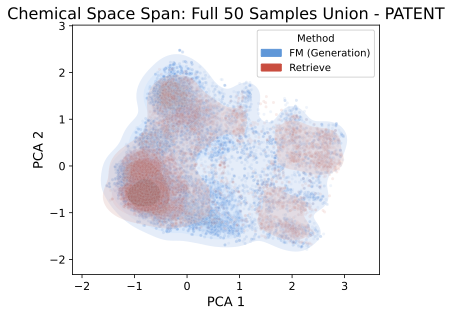} 
        \caption{PMV17 dataset.}
    \end{subfigure}
    
    \caption{Visualizations of the chemical space explored by our foundation model MMPT-FM (blue) versus Database Retrieval (red) on (a) ChEMBL and (b) PMV17 datasets. 
    }
    \label{fig:space_span}
\end{figure}

\begin{table}[t]
\centering
\caption{Effect of beam size on average validity of MMPT-FM, averaged on the ChEMBL MMPT held-out test set.}
\label{tab:beam_validity}
\setlength{\tabcolsep}{4pt}
\renewcommand{\arraystretch}{1.05}
\small
\begin{tabular}{lccccc}
\toprule
\textbf{Beam} & 400 & 600 & 800 & 1000 & 1200 \\
\midrule
\textbf{Avg Validity} & 0.9992 & 0.9991 & 0.9989 & 0.9988 & 0.9986 \\
\bottomrule
\end{tabular}
\vspace{-2mm}
\end{table}

To assess the fundamental generative capability of MMPT-FM, we evaluate its ability to explore the chemical space compared to Database Retrieval. We utilize a randomly sampled gruop of 50 unique variables from both ChEMBL and PMV17 to generate candidate sets. The structural distributions are visualized by projecting their Morgan fingerprints using PCA, comparing the spans of FM-generated candidates against retrieved variables.
Visual analysis of the PCA projections (Fig.~\ref{fig:space_span}) reveals that MMPT-FM (blue) consistently explores a substantially larger chemical volume compared to Database Retrieval (red), demonstrating superior extrapolation beyond the training distribution.  

We further study how beam size affects the validity of MMPT generation on the ChEMBL MMPT held-out test set. As shown in Table~\ref{tab:beam_validity}, validity remains consistently high across all beam sizes, with only a slight decrease from 0.9992 at beam 400 to 0.9986 at beam 1200, which is usually enough for application scenarios. Notably, a beam size of 1200 significantly exceeds the depth typically required for practical medicinal chemistry applications. This indicates that our generator is not validity-bottlenecked by search depth within its intended operational range; increasing beam primarily expands the candidate pool but does not harm chemical correctness.

\subsubsection{Evaluation of the Prompted Generation Mechanism of MMPT-FM}\label{sec:eval_prompt}
We evaluate the prompted generation capability of MMPT-FM via a masked infilling task. This task motivates the assessment of whether the model can strictly adhere to user-specified structural templates while proposing chemically plausible completions. We construct a benchmark by randomly sampling 50 unique variables from ChEMBL and PMV17 with lengths exceeding 15 characters. For each variable, three masked versions are generated by masking a consecutive sequence of tokens with a length capped at $\min(\text{half of the output}, 8)$. Here prompted generation is performed using 1,000 beams and a length margin of 7.
Results are given in Table~\ref{tab:prompt_eval}. At $K=1$, the model achieves near-perfect validity and high GT recovery (58.0\% for ChEMBL, 46.0\% for PMV17). By $K=20$, the model attains near-perfect recall across both datasets. Furthermore, at $K=200$, the model produces a significant number of unique valid candidates (41.6 for ChEMBL, 31.6 for PMV17), confirming its ability to explore diverse chemical spaces even under rigid structural constraints, which shows the effectiveness of our prompted generation mechanism in generating promising, valid and user-desired variables.

\begin{table}[t]
\centering
\small
\vspace{-3mm}
\caption{Prompted generation on ChEMBL and PMV17 at different numbers of generations ($K$). We report Validity, Recall of ground truth, and numbers of generated unique and valid variables (\#Unique).}
\label{tab:prompt_eval}
\vspace{-3mm}
\begin{tabular}{c c c c c c c}
\toprule
& \multicolumn{3}{c}{ChEMBL} & \multicolumn{3}{c}{PMV17} \\
\cmidrule(lr){2-4}\cmidrule(lr){5-7}
$K$ & Valid (\%) & Recall (\%) & \#Uniq & Valid (\%) & Recall (\%) & \#Uniq \\
\midrule
1   & 100.0 & 58.0  & 1.00  & 98.0  & 46.0  & 0.98 \\
10  & 86.0  & 96.0  & 8.60  & 79.2  & 90.0  & 7.92 \\
20  & 74.1  & 100.0 & 14.82 & 62.9  & 96.0  & 12.58 \\
50  & 48.0  & 100.0 & 24.00 & 39.24 & 98.0  & 19.62 \\
100 & 31.7  & 100.0 & 31.70 & 24.84 & 98.0  & 24.84 \\
200 & 20.8  & 100.0 & 41.60 & 15.79 & 98.0  & 31.58 \\
\bottomrule
\end{tabular}
\end{table}

\subsubsection{Analysis of Ditribution Steering via RAG. }\label{sec:eval_rag}
\begin{figure}[t]
    \centering
    \vspace{-3mm}
    \includegraphics[width=0.8\linewidth, trim={0 0 0 16mm}, clip]{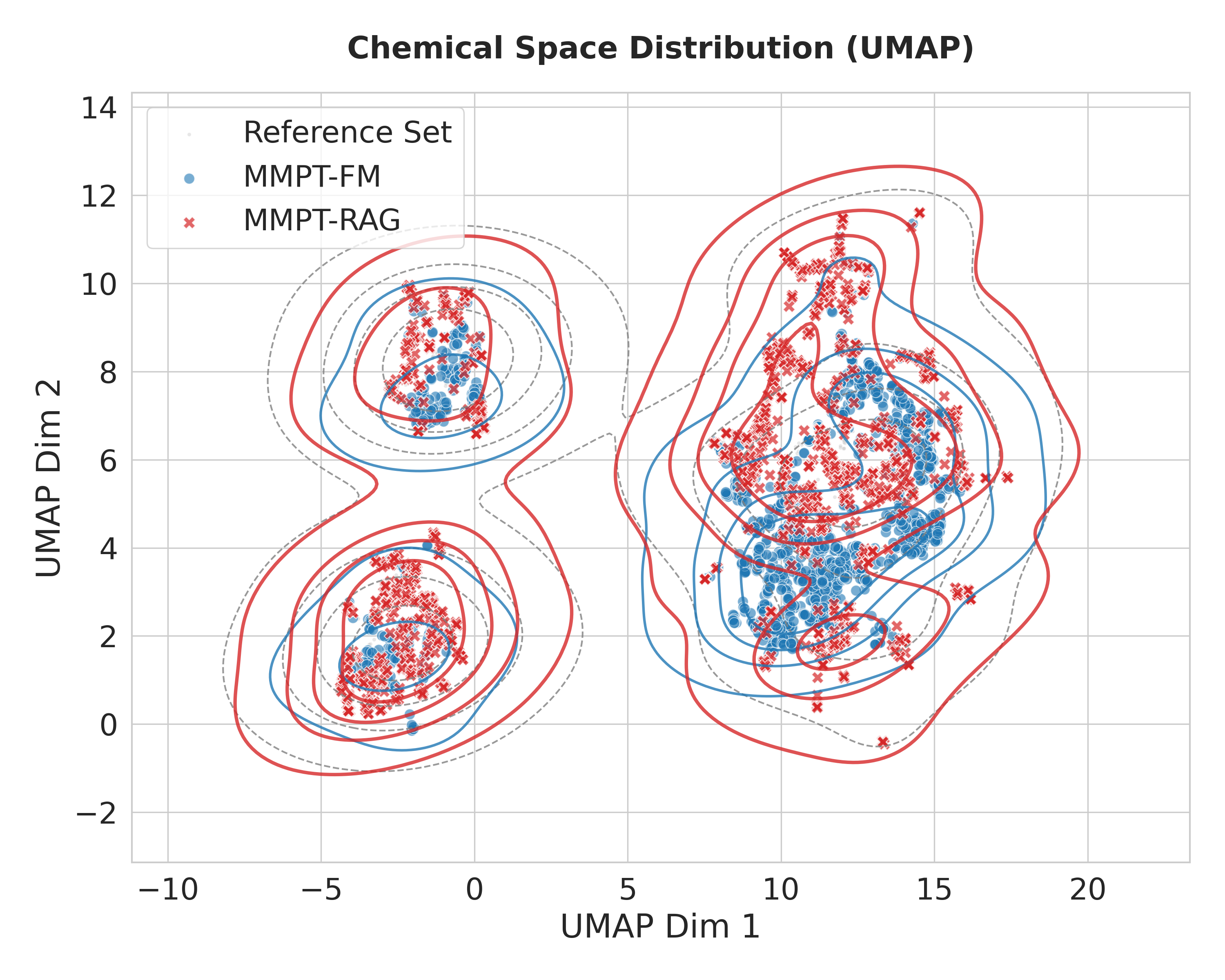}
    \vspace{-3mm}
    \caption{UMAP visualization of MMPT-FM and MMPT-RAG's chemical landscape on PMV17. The grey shaded areas represent the reference dataset's distribution. Compared to FM inference (blue), the MMP-RAG framework (red) populates structural voids where the foundation model is sparse or absent.}
    \vspace{-3mm}
    \label{fig:rag_dist}
\end{figure}
To investigate how retrieval augmentation steers the generative process toward target chemical domains, we visualize the global distribution of generated analogs against the reference patent dataset (PMV17). We compare the union of outputs from 50 unique inputs generated via vanilla FM inference versus the MMPT-RAG framework, using the PMV17 dataset as the reference manifold represented by the grey shaded regions in Fig.~\ref{fig:rag_dist}. For clarity, the RAG visualization highlights the additional coverage contributed beyond vanilla FM outputs, illustrating the complementary effect of retrieval augmentation.

As illustrated in Fig.~\ref{fig:rag_dist}, MMPT-RAG (red) expands into multiple structural regions that remain underexplored by the vanilla foundation model (blue). The vanilla FM tends to concentrate in high-probability general regions, leaving several reference clusters sparsely covered. In contrast, retrieval augmentation encourages the model to populate these underrepresented areas, effectively filling distributional gaps. This shift indicates that RAG enhances coverage of project-relevant chemical space by guiding generation toward regions of the reference dataset.

\subsubsection{Hyperparameter Sensitivity Analysis. }\label{sec:hyperparam}

To better understand the behavior of MMPT-RAG, we perform a sensitivity analysis on three hyperparameters that control generation quality: the number of retrieved clusters expanded, i.e., generate using its MCS (default 10), the number of variables generated per cluster (default 50), and the mask-infilling length range used during sequence completion (default = original length before masking $\pm 7$). We vary one hyperparameter at a time while keeping the others fixed: clusters ${3,5,10,20}$, variables per cluster ${10,25,50,100}$, and mask range ${[1,3],[1,5],[1,7],[1,9]}$.
Figure~\ref{fig:hyperparam} shows that moderate increases in all three parameters improve performance, while larger values provide diminishing returns. Expanding more clusters improves results by exposing more diverse retrieved patterns and then stabilizes around 10 to 20 clusters. Increasing the number of variables per cluster expands search coverage but saturates near 50 samples. Widening the mask-infilling range consistently helps up to $[1,7]$, after which additional flexibility yields little gain. Based on these trends, we recommend 10 clusters, 50 variables per cluster, and a mask range of $[1,7]$ to balance performance and computation. 
\begin{figure}[t]
    \centering
    \includegraphics[width=1.\linewidth]{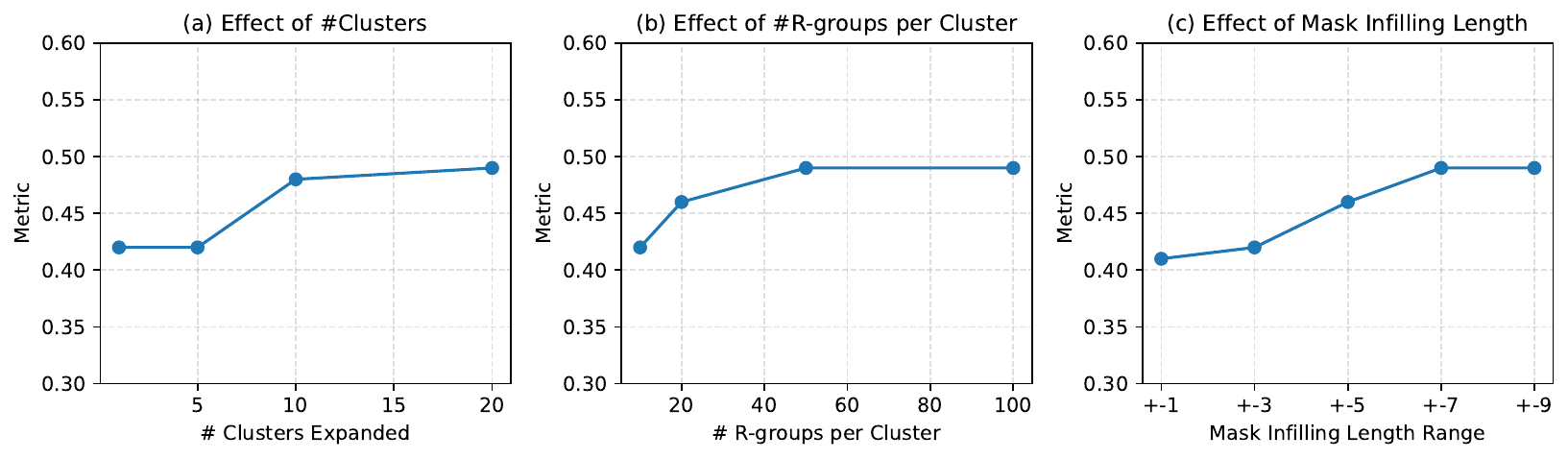}
    \vspace{-3mm}
    \caption{Hyperparameter Study. (a) the number of clusters to expand, (b) the number of variables to generate for each cluster, (c) the range of mask length to fill.}
    \label{fig:hyperparam}
\end{figure}

\subsection{Qualitative Evaluation}\label{sec:qualitative}

\begin{figure}[t]
    \centering
    \vspace{-3mm}
    \includegraphics[width=0.8\linewidth, trim={0 70mm 0 0}, clip]{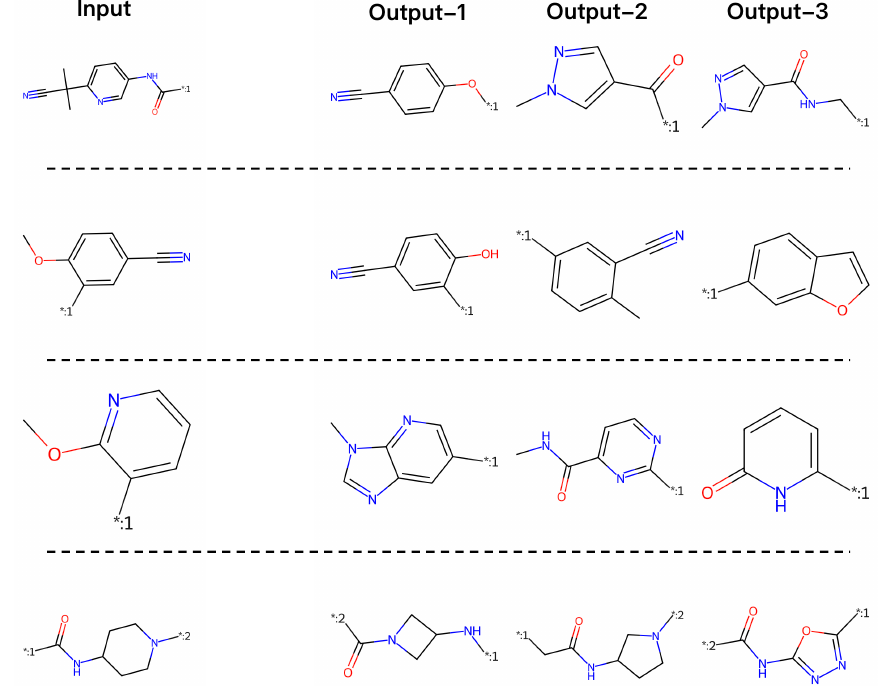}
    \caption{Examples of MMPT-FM generations. In each row, the left structure is the input variable, and the structures on the right are generated outputs.}
    \label{fig:vis_fm}
\end{figure}

\begin{figure}[t]
    \centering
    \includegraphics[width=0.9\linewidth, trim={0 0 0 0mm}, clip]{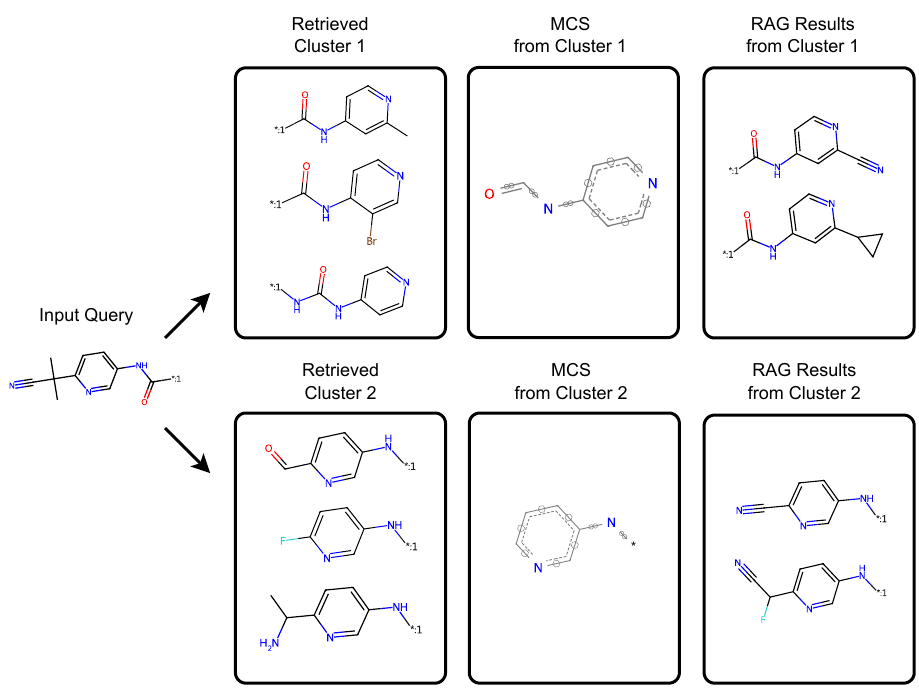}
    \caption{Examples of MMPT-RAG generations. Retrieved variables are clustered, an MCS is extracted per cluster, and generation is conditioned on each cluster's template (MCS).}
    \vspace{-3mm}
    \label{fig:vis_rag}
\end{figure}

To better understand the behavior of the proposed model beyond quantitative metrics, we present representative qualitative examples of generated variables. Fig.~\ref{fig:vis_fm} illustrates variables directly generated from MMPT-FM. For each input variable (left column of each row), the model produces multiple diverse and chemically plausible variants. The generated variables preserve valid valence patterns and maintain realistic functional groups and ring systems. Notably, the model naturally supports multiple attachment points, demonstrating its ability to handle context-dependent transformations and generate structurally coherent edits across different substitution sites. 

Fig.~\ref{fig:vis_rag} shows MMPT-RAG's generation results for the same input as the first example in Fig.~\ref{fig:vis_fm}. In this case, we first retrieve structurally similar historical variables from the database, then cluster the retrieved examples based on shared substructures. Two representative clusters are shown. For each cluster, we show theior Maximum Common Substructure (MCS), which serves as a structural template capturing the dominant transformation pattern within that cluster.
Conditioned on these cluster-specific templates, the generated outputs reflect the characteristic patterns of their corresponding clusters while remaining chemically valid and diverse. Compared to the standalone foundation model, RAG effectively steers outputs toward specific transformation families while preserving chemical plausibility.

\section{Conclusion}

In this work, we presented a paradigm shift in generative molecular design by reframing analog generation as a transformation-to-transformation task grounded in Matched Molecular Pair Transformations (MMPTs). Unlike traditional molecule-level approaches that often lack localized control, our framework explicitly models the precise chemical edits that define medicinal chemistry intuition.
By training a foundation model (\textbf{MMPT-FM}) on large-scale transformation data, we achieved scalable generation of variable substructures that balances chemical plausibility with structural novelty. To address the specific constraints of active drug discovery projects, we introduced \textbf{MMPT-RAG}, a retrieval-augmented framework that leverages external reference datasets to steer generation toward relevant, project-specific motifs. Our extensive evaluation on both general chemical corpora and time-split patent series demonstrates that this approach not only improves diversity and validity but effectively recovers prospective ligands in realistic discovery scenarios. Ultimately, this framework operationalizes MMPTs as a first-class generative abstraction, offering a powerful tool to synergize machine learning capability with human expertise.

\section{Limitations and Ethical Considerations}
Our approach relies on the availability and coverage of large historical transformation datasets, and its performance may vary in underrepresented chemical domains. Our framework is intended for research use, and does not introduce specific ethical concerns.

\newpage
\clearpage
\bibliographystyle{ACM-Reference-Format}
\bibliography{reference}

\newpage
\clearpage

\appendix
\section{Implementation Details}\label{append:implementation}

\textbf{Foundation model training and inference.} To obtain the foundation model, we leverage the ChemT5 model \cite{christofidellis2023unifying}, an encoder-decoder transformer pretrained on chemical datasets, as our base model. ChemT5 contains approximately 220 million parameters and has been fine-tuned for tasks in cheminformatics. It consists of 12 layers, 12 attention heads, 220M parameters, and processes input sequences up to 512 tokens \cite{christofidellis2023unifying}.
We employed standard supervised training to fine-tune all parameters of the base model. Teacher forcing \cite{williams1989learning} is incorporated to improve training stability. The training was conducted with a batch size of 64 on each device, and a learning rate of 5e-4. We use an early stop strategy with a tolerance of 2 epochs based on the evaluation loss. Utilizing four NVIDIA A6000 GPUs (48 GB each), the training process required approximately 70 hours to complete.
During inference with the foundation model, for each input, we use beam search to produce 1000 outputs with a maximum length of 50.

\textbf{Retriever.} 
We pre-build a nearest-neighbor index with HNSW \cite{malkov2018efficient} and query it with cosine distance over the Morgan Fingerprint \cite{morgan1965generation} embedding space. At inference time, we first retrieve at most top 500 nearest input variables, then expand each input into its associated label set. To ensure compatibility with the query, we filter candidates by the number of wild atoms (i.e., attachment points) to match the query variable. We then compute Morgan fingerprints and re-rank retrieved label candidates by Tanimoto similarity between the query and candidate labels, retaining the top set for downstream RAG steps.

\textbf{Clustering, Template Construction.} We cluster the retrieved labels in structure space using a shared-substructure clustering method, where we first compute pairwise similarities between retrieved outputs using the size of the RDKit maximum common substructure (MCS) normalized by the smaller molecule. We then perform agglomerative hierarchical clustering (average linkage) on the corresponding distance matrix, and cut the dendrogram at a similarity threshold of 0.70 to obtain clusters. To avoid overly coarse motifs, any cluster exceeding 10 molecules is recursively split using the same linkage procedure until all clusters satisfy the size constraint. For each retained cluster, we compute a Maximum Common Substructure (MCS), whose resulting SMARTS string serves as the cluster-invariant template.

We further convert this template into a partially specified output constraint by masking atoms outside the invariant scaffold. Concretely, we apply substructure masking to produce a template string where masked spans are denoted by a special masking character. We then convert each masked span into a single <BLANK> placeholder and perform span infilling using the generative model, as introduced in the prompted generation section.

\textbf{Prompted Generation via Mask Infilling.}
The infilling search is controlled by 11 maximum new tokens per blank, 200 maximum total candidate continuations, and 200 top candidates scored. The final RAG outputs are the union of valid, RDKit-parsable infilled candidates across templates. We exclude any duplicate sequences.

\textbf{Implementation of Compared Methods.}
For Database Retrieval, we first retrieve at most 50 input variables from the reference dataset that are most similar to the test query based on fingerprint similarity. We then collect all corresponding output variables paired with these retrieved inputs. If the resulting candidate set exceeds 1000 outputs, we retain the 1000 variables that are most similar to the test query to ensure a fair comparison in terms of candidate size. REINVENT4 (LibINVENT) follows a different formulation, generating variables conditioned on a scaffold rather than directly modeling MMPT-style variable-to-variable transformations. To align the setting with our task, we use the constant fragment identified by MMPDB as the scaffold input and generate up to 1000 candidate variables for each query. 

For consistency, all methods, including MMPT-FM and MMPT-RAG, generate 1000 candidates per input.

\section{Proof of Theorem~\ref{thm:steering}}\label{append:proof}
\begin{proof}
According to the workflow defined in Eq.~\ref{eq:rag-gen}, the MMPT-RAG framework constructs the global generation as a mixture of cluster-conditioned distributions with weights $\tilde{\pi}_k$. Summing over the variables $y$:
\begin{align*}
p_{\text{RAG}}(y \mid x) &= \sum_{k=1}^K \tilde{\pi}_k p(y \mid x, T_k) \\
&= \sum_{k=1}^K \tilde{\pi}_k \big[ (1 - \alpha_k) p_{\theta}(y \mid x) + \alpha_k p(y \mid T_k) \big] \\
&= p_{\theta}(y \mid x) \sum_{k=1}^K \tilde{\pi}_k (1 - \alpha_k) + \sum_{k=1}^K \tilde{\pi}_k \alpha_k p(y \mid T_k).
\end{align*}
Using $\sum \tilde{\pi}_k = 1$ and the definition $\bar{\alpha} = \sum \tilde{\pi}_k \alpha_k$, the first term simplifies to $(1 - \bar{\alpha}) p_{\theta}(y \mid x)$. The second term, by multiplying and dividing by $\bar{\alpha}$, recovers the effective reference $p_{\text{ref}}^*(y \mid x)$. Thus:
\[
p_{\text{RAG}}(y \mid x) = (1 - \bar{\alpha}) p_{\theta}(y \mid x) + \bar{\alpha} p_{\text{ref}}^*(y \mid x).
\]
\end{proof}

\end{document}